\documentclass[10pt]{article} 
\usepackage[final]{neurips_2025}
\usepackage[normalem]{ulem}
\usepackage{enumitem}

\usepackage{amsmath,amsfonts,bm}

\def\dd{{\mathrm{d}}}
\def\sX{{\mathsf{X}}}
\def\sY{{\mathsf{Y}}}
\def\sZ{{\mathsf{Z}}}

\def\1{\bm{1}}

\DeclareMathAlphabet{\mathsfit}{\encodingdefault}{\sfdefault}{m}{sl}
\SetMathAlphabet{\mathsfit}{bold}{\encodingdefault}{\sfdefault}{bx}{n}

\newcommand{\R}{\mathbb{R}}

\newcommand{\KL}{D_{\mathrm{KL}}}

\DeclareMathOperator*{\argmin}{arg\,min}

\usepackage{xcolor}
\usepackage{framed, multido}
\usepackage{tikz}

\usepackage[utf8]{inputenc}
\usepackage{lmodern}
\usepackage[T1]{fontenc}    
\usepackage{booktabs}       
\usepackage{nicefrac}       
\usepackage{microtype}      
\usepackage{lipsum}
\usepackage{adjustbox}

\usepackage{amsmath,amsthm,amssymb, amsfonts}
\usepackage{pythonhighlight}
\usepackage{setspace}
\usepackage{svg} 
\usepackage{float}
\usepackage{enumerate, setspace, xcolor}
\usepackage{soul}
\sethlcolor{yellow}
\usepackage{subcaption}
\usepackage{algorithm}
\usepackage{algpseudocode}
\usepackage{natbib}
\usepackage[makeroom]{cancel}

\definecolor{astral}        {RGB}{46,116,181}
\definecolor{cb-blue}       {RGB}{70, 130, 180}
\definecolor{orange}        {RGB}{214,150, 92}
\definecolor{green}         {RGB}{136,196,136}

\usepackage{hyperref}
\hypersetup{
colorlinks=true,
allcolors=cb-blue
}
\usepackage{cleveref}
\usepackage{bookmark}
\svgpath{{figures/}}
\usepackage{subcaption}

\makeatletter
\newcommand{\HEADER}[1]{\ALC@it\underline{\textsc{#1}}\begin{ALC@g}}
\newcommand{\ENDHEADER}{\end{ALC@g}}
\makeatother

\theoremstyle{plain}
\newtheorem{theorem}{Theorem}[section]
\newtheorem{lemma}[theorem]{Lemma}

\theoremstyle{definition}

\newtheorem{proposition}[theorem]{Proposition}

\usepackage[acronym]{glossaries}

\newacronym{DM}{{{\textsc{\small DM}}}}{Diffusion Model}
\newacronym{GAN}{{{\textsc{\small GAN}}}}{Generative Adversarial Network}
\newacronym{VAE}{{{\textsc{\small VAE}}}}{Variational Autoencoder}
\newacronym{IWAE}{{{\textsc{\small IWAE}}}}{Importance Weighted Variational Autoencoder}
\newacronym{DiffusionVAE}{{{\textsc{\small DiffusionVAE}}}}{ VAE with a Diffusion Prior}
\newacronym{VI}{{{\textsc{\small VI}}}}{Variational Inference}
\newacronym{PSVI}{{{\textsc{\small PSVI}}}}{Particle Semi-implicit Variational Inference}
\newacronym{UIVI}{{{\textsc{\small UIVI}}}}{Unbiased Implicit Variational Inference}
\newacronym{SVI}{{{\textsc{\small SVI}}}}{Semi-implicit Variational Inference}
\newacronym{VQVAE}{{{\textsc{\small VQ-VAE}}}}{Vector-Quantized Variational Autoencoder}
\newacronym{LDM}{{{\textsc{\small LDM}}}}{Latent Diffusion Model}
\newacronym{EBM}{{{\textsc{\small EBM}}}}{Energy-based Model}
\newacronym{IPLD}{{{\textsc{\small IPLD}}}}{Interacting Particle Latent Diffusion}
\newacronym{DDPM}{{{\textsc{\small DDPM}}}}{Denoising Diffusion Probabilistic Model}
\newacronym{LSGM}{{{\textsc{\small LSGM}}}}{Latent Score-based Generative Models}
\newacronym{SVGD}{{{\textsc{\small SVGD}}}}{Stein Variational Gradient Descent}
\newacronym{MMD}{{{\textsc{\small MMD}}}}{Maximum Mean Discrepancy}
\newacronym{ELBO}{{{\textsc{\small ELBO}}}}{Evidence Lower BOund}
\newacronym{LDDM}{{{\textsc{\small LDDM}}}}{Latent Denoising Diffusion Model}
\newacronym{LVM}{{{\textsc{\small LVM}}}}{Latent Variable Model}
\newacronym{PGD}{{{\textsc{\small PGD}}}}{Particle Gradient Descent}
\newacronym{MPGD}{{{\textsc{\small MPGD}}}}{Momentum Particle Gradient Descent}
\newacronym{EM}{{{\textsc{\small EM}}}}{Expectation-Maximisation}
\newacronym{FID}{{{\textsc{\small FID}}}}{Fréchet Inception Distance}
\newacronym{GMM}{{\textsc{\small GMM}}}{Gaussian Mixture Model}
\newacronym{SWD}{{\textsc{\small SWD}}}{Sliced Wasserstein Distance}
\newacronym{LPIPS}{\textsc{\small LPIPS}}{Learned Perceptual Image Patch Similarity}
\newacronym{PSNR}{\textsc{\small PSNR}}{Peak Signal-to-Noise Ratio}
\newacronym{WGF}{\textsc{\small WGF}}{Wasserstein Gradient Flow}
\newacronym{DWGF}{\textsc{\small DWGF}}{Diffusion-regularized Wasserstein Gradient Flow}
\newacronym{PSLD}{\textsc{\small PSLD}}{Posterior Sampling with Latent Diffusion}
\newacronym{RLSD}{\textsc{\small RLSD}}{Repulsive Latent Score Distillation}

\newcommand{\vae}{\gls*{VAE}}

\newcommand{\mmd}{\gls*{MMD}}

\title{A Gradient Flow Approach to Solving Inverse Problems with Latent Diffusion Models}
\author{%
   Tim Y. J. Wang \\
  Department of Mathematics\\
  Imperial College London\\
  \texttt{tw1320@ic.ac.uk}
  \And
  {O}. Deniz Akyildiz \\
  Department of Mathematics\\
  Imperial College London\\
  \texttt{deniz.akyildiz@imperial.ac.uk}
}

\begin{document}

\maketitle
\begin{abstract}
Solving ill-posed inverse problems requires powerful and flexible priors. We propose leveraging pretrained latent diffusion models for this task through a new training-free approach, termed Diffusion-regularized Wasserstein Gradient Flow (DWGF). Specifically, we formulate the posterior sampling problem as a Wasserstein gradient flow in the latent space of an expected negative log posterior objective, regularized by a Kullback-Leibler divergence to the diffusion prior. We demonstrate the performance of our method on standard benchmarks using StableDiffusion \citep{rombachhighresolutionimagesynthesis2022} as the prior.
\end{abstract}

\section{Introduction}
Inverse problems are ubiquitous in science and engineering. They involve finding the underlying true signal $x_0$ from the corrupted observation $y$. In this work, we are primarily concerned with inverse problems of the form:
\begin{equation}\label{eq:simple_inv}
    y=\mathcal{A}(x_0)+\epsilon, \qquad \epsilon \sim \mathcal{N}(0, \sigma_y^2 I),
\end{equation}
where $\mathcal{A}: \sX\to \sY$ is a known forward corruption operator and $\epsilon$ is the additive Gaussian noise. Through the lens of Bayesian inference \citep{stuart2010inverse}, the solution to the problem in \eqref{eq:simple_inv} can be elegantly viewed as sampling from the posterior $p(x_0|y) \propto p(y|x_0) p_{data}(x_0)$ by placing a prior on $x_0$. While traditional hand-crafted priors are often chosen for mathematical convenience, they struggle with high-dimensional, ill-posed inverse problems. 

In recent years, alternative priors based on diffusion models \citep{pmlr-v37-sohl-dickstein15,hodenoisingdiffusionprobabilistic2020, songscorebasedgenerativemodeling2021} have become the state-of-the-art for solving imaging inverse problems. A common strategy is to modify the sampling process of a \emph{pixel-space} diffusion model to guide the sampler towards the posterior $p(x_0|y)$ through various techniques, for instance, proximal \citep{zhuDenoisingDiffusionModels2023, wuDiffusionPosteriorProximal2024}, gradient-based \citep{chung2023diffusion, song2023pseudoinverseguided, boys2024tweedie}, variational inference \citep{mardaniVariationalPerspectiveSolving2023,zilbersteinRepulsiveLatentScore2024}, or sequential Monte-Carlo methods \citep{cardoso2024monte, dou2024diffusion,chenSolvingInverseProblems2025}.

However, adapting these techniques to more computationally efficient \emph{latent-space} diffusion models \citep{rombachhighresolutionimagesynthesis2022} is not straightforward, as the model defines a prior $p(z_0)$ on a latent variable instead of the ground truth signal $x_0$. Approaches such as \citet{routSolvingLinearInverse2023,songSolvingInverseProblems2024} adapt diffusion posterior sampling \citep{chung2023diffusion} to the latent space with data consistency regularization, while \citet{zilbersteinRepulsiveLatentScore2024} builds on RED-Diff \citep{mardaniVariationalPerspectiveSolving2023} and simulate a particle system for an augmented distribution $q(x_0, z_0|y)$ defined on the product of pixel and latent space.

In this work, we depart from adapting existing pixel-space methods. Instead, we formulate the posterior sampling problem from first principles as a Wasserstein gradient flow directly in the latent space. Using the diffusion prior, we develop a regularized gradient flow, which we term \gls*{DWGF}. We derive a system of ordinary differential equations (ODEs) that approximates this flow, providing a principled method for solving inverse problems with latent diffusion priors.

\paragraph{Notations.} We use $\mathsf{X} = \mathbb{R}^{d_x}$ to denote the ambient (pixel) space, $\mathsf{Y}\subseteq \R^{d_x}$ the observation space, and $\mathsf{Z} = \mathbb{R}^{d_z}$ the latent space. We denote the space of probability measures on $\mathbb{R}^d$ with finite $q$-th moments as $\mathcal{P}_q(\mathbb{R}^d)$; in this work, we focus on the cases where $q=2$. We use $\delta\mathcal{F}[\mu]/\delta\mu$ to denote the first $L^2$-variation of functional $\mathcal{F}$ at $\mu$.

\section{A Wasserstein Gradient Flow Approach}
\label{sec:method}
Assume that we are given a pretrained latent diffusion model, which consists of a decoder $p_{\phi^-}(x_0|z_0)$, encoder $r_{\phi^-}(z_0 | x_0)$, and a diffusion generative model $p_{\theta^-}(z_0)$ in the latent space. Given the inverse problem defined in \eqref{eq:simple_inv}, we 
would like to leverage this model to obtain an approximate posterior distribution, denoted here
\begin{align}\label{eq:semi_implicit_posterior}
q_\mu(x_0 | y) = \int p_{\phi^-}(x_0|z_0)\mu(z_0|y)\dd z_0.
\end{align}
In order to obtain this distribution, one needs to approximate the unknown distribution $\mu(z_0 | y)$. We thus consider the following regularized optimization problem over $\mathcal{P}_2(\sZ)$:
\begin{equation}\label{eq:min_functional}
    \mu_\star(z_0 | y) \in \argmin_{\mu \in\mathcal{P}_2(\sZ)} \mathcal{F}[\mu(z_0 | y)] + \gamma \mathcal{R}[\mu(z_0|y)],
\end{equation}
where $\mathcal{F}[\mu] := -\mathbb{E}_{q_\mu(x_0|y)}\left[\log p(x_0|y)\right]$ is the expected negative log-posterior of the decoded measure $q_\mu(x_0|y)$ and $\mathcal{R}:\mathcal{P}_2(\sZ)\to \R_+$ can be taken as any regularization that makes the minimizers of the gradient flow well-defined \citep{crucinioSolvingFredholmIntegral2024}, and $\gamma>0$ controls the strength of regularization. In order to solve the problem in \eqref{eq:min_functional}, we adopt a gradient descent approach in the space of probability measures . 
The gradient flow of any functional $\mathcal{L}[\mu]$ on $\mathcal{P}_2(\sZ)$ starting at some $\bar{\mu}_0$ is given by \citep{figalli2021invitation}: 
\begin{equation}\label{eq:transport_eq_main}
\frac{\partial\mu_t}{\partial t} = \nabla_{z_0}\cdot \left(\mu_t \nabla_{z_0} \frac{\delta\mathcal{L}[\mu]}{\delta \mu}\right), \qquad \mu_0 = \bar{\mu}_0.
\end{equation}
The PDE above is the continuity equation corresponds to the following ODE \citep[Chapter 8]{ambrosio2005gradient}: 
\begin{equation}\label{eq:ode_flow_main}
    \frac{\dd z_{0,t}}{\dd t} = - \nabla_{z_0} \frac{\delta\mathcal{L}[\mu]}{\delta \mu} = - \left(\nabla_{z_0} \frac{\delta\mathcal{F}[\mu]}{\delta \mu} + \gamma \nabla_{z_0} \frac{\delta\mathcal{R}[\mu]}{\delta \mu}\right), \qquad z_{0,0}  \sim \bar{\mu}_0,
\end{equation}
 where we have set $\mathcal{L}[\mu] := \mathcal{F}[\mu] + \gamma \mathcal{R}[\mu]$. The solution to \eqref{eq:ode_flow_main} defines a flow map $T_t(z_{0,0})=z_t$ such that $\mu_t = (T_t)_\# \bar{\mu}_0$. To simulate \eqref{eq:ode_flow_main}, we need to compute the first variation ${\delta\mathcal{L}[\mu]}/{\delta \mu}$ and its gradient. In what follows, we will construct the regularization and derive the first variations.

\subsection{\texorpdfstring{Wasserstein gradient of $\mathcal{F}[\mu]$}{Wasserstein gradient of F}}
Substituting~\eqref{eq:semi_implicit_posterior} into the definition of $\mathcal{F}$ and exchanging the order of integration shows that $\mathcal{F}$ is linear in $\mu$, namely $\mathcal{F}[\mu]=\int V(z_0)\mu(z_0|y)\dd z_0$ with $V(z_0):=-\mathbb{E}_{p(x_0|z_0)}[\log p(x_0|y)]$. The first variation of a linear functional is its integrand (see \Cref{app:grad_flow_derivation_F}), hence to obtain the first term of the drift in the ODE \eqref{eq:ode_flow_main}, we can compute 
\begin{align}\label{eq:F_gradient_of_first_variation}
\nabla_{z_0} \frac{\delta \mathcal{F}[\mu]}{\delta \mu} = -\mathbb{E}_{\epsilon\sim \mathcal{N}(0,I)}\left[
    \nabla_{x_0} \log p(g_{\phi^-}(\epsilon, z_{0})|y) \frac{\partial \mathcal{D}_{\phi^-}(z_{0})}{\partial z_{0}}
    \right],
\end{align}%
where we have used the chain rule and the reparameterization trick $g_{\phi^-}(\epsilon,z_0):=\mathcal{D}_{\phi^-}(z_0)+\rho \epsilon$ for the Gaussian decoder.

We note that the only intractable term in \eqref{eq:F_gradient_of_first_variation} is the posterior score $\nabla_{x_0}\log p(x_0|y)$, which admits the decomposition $\nabla_{x_0}\log p(x_0|y)=\nabla_{x_0} \log p(y|x_0)+\nabla_{x_0}\log p(x_0)$. While the first term $\nabla_{x_0} \log p(y|x_0)$ is tractable for being the gradient of the Gaussian likelihood \eqref{eq:simple_inv}, the data score $\nabla_{x_0}\log p(x_0)$ requires approximation. Assuming regularity conditions that permit interchanging the gradient and integral, we can express the prior score as an expectation, which we approximate using the encoder of the pretrained latent diffusion model:
\begin{align}\label{eq:data_score}
    \nabla_{x_0} \log p(x_0) 
    &= \int_\sZ [\nabla_{x_0} \log p(x_0|z_0)] p(z_0|x_0) \dd z_0 \approx \int_\sZ [\nabla_{x_0} \log p(x_0|z_0)] \tilde{r}_{\phi^-}(z_0|x_0) \dd z_0,
\end{align}
where $\tilde{r}_{\phi^-}(z_0|x_0)$ is the approximate posterior distribution given by a \gls*{VAE} \citep{kingma2013autoencodingvb}, which is part of the pretrained latent diffusion model. 

\subsection{Wasserstein gradient of \texorpdfstring{$\mathcal{R}[\mu(z_0 | y)]$}{R}}
To fully leverage the information provided by the generative model, we first construct the prior regularization term $\mathcal{R}(\mu(z_0|y))$ based on the pretrained latent diffusion model. Analogous to \cite{wangProlificDreamerHighFidelityDiverse2023,luoDiffinstructUniversalApproach2024}, we consider a weighted KL divergence along the diffusion process:
\begin{equation}\label{eq:weighted_KL}
    \mathcal{R}[\mu(z_0|y)] := \KL^{w, [0,T]}(\mu(z_0|y)\| p_{\theta^-}(z_0))=\int_0^T w(s) \KL(\mu(z_s|y)\| p_{\theta^-}(z_s)) \dd s,
\end{equation}
where $w(s):[0,T]\to \R_+$ is a time-dependent weighting term and the densities $\mu(z_s|y):=\int_\sZ \mathcal{N}(z_s; \alpha_s z_0, \sigma_s^2I) \mu(z_0|y) \dd z_0$ are pushforwards of $\mu$ through the forward transition kernel of the diffusion model. We remark that the weighted KL divergence in \eqref{eq:weighted_KL} is a well-defined functional and admits favorable properties (cf. \Cref{app:proof_thm_kl}) as summarized below.
\begin{theorem}\label{thm:weighted_kl_properties}
    The weighted KL divergence $\KL^{w, [0,T]}(\mu(z_0|y)\| p_{\theta^-}(z_0))$ \eqref{eq:weighted_KL} is i) nonnegative, ii) convex in the first component $\mu(z_0|y)$, and iii)  is minimized if and only if the standard KL divergence $\KL(\mu(z_0|y)\|p_{\theta^-}(z_0))$ is minimized.
\end{theorem}
Accordingly, we obtain the gradient of the first variation of $\mathcal{R}$ (cf. \Cref{app:grad_flow_derivation_R}):
\begin{align}\label{eq:R_gradient_of_first_variation}
\nabla_{z_0} \frac{\delta \mathcal{R}[\mu]}{\delta \mu} = \mathbb{E}_{s,\epsilon} \left[\tilde{w}(s)
        \left(\nabla_{z_s} \log \int_\sZ p(z_{s}|z_{0}) \mu(z_{0}|y) \dd z_{0} - \nabla_{z_s} \log p_{\theta^-}(z_{s})\right) \frac{\partial z_{s}}{\partial z_{0}}
    \right],
\end{align}
where the expectation is taken over $s\sim \mathcal{U}(0,T), \ \epsilon\sim \mathcal{N}(0,I)$, we set $\tilde{w}(s):=Tw(s)$, and $p(z_s|z_0)=\mathcal{N}(z_s;\alpha_s z_0, \sigma_s^2I)$ is the forward kernel of the diffusion model. To approximate the integral in \eqref{eq:R_gradient_of_first_variation}, we use a particle-based approach to form a Monte-Carlo approximation using $\{z_{0}^{(i)}\}_{i} \sim \mu(z_{0}|y)$ \citep{wangProlificDreamerHighFidelityDiverse2023,limParticleSemiimplicitVariational2024}.

\subsection{Algorithmic Considerations}
\label{sec:algo_consi}
\textbf{Final Gradient Flow.} We can simply combine the gradients derived in \eqref{eq:F_gradient_of_first_variation} and \eqref{eq:R_gradient_of_first_variation} to obtain the final gradient flow in \eqref{eq:ode_flow_main}. As we have an intractable term, we will simulate $N$ identical copies of this ODE, where the integral in \eqref{eq:R_gradient_of_first_variation} is approximated using these particles as mentioned. We mention some practical aspects of this implementation below.

\paragraph{Deterministic Encoding} During our experiments, we observe that the diagonal log covariance matrix of the approximate Gaussian posterior only contains small values $(\approx -17)$, thus we may view the encoding process as deterministic and take $\mathcal{E}_{\phi^-}:\sX\to\sZ$ as a map to the mean of the Gaussian. In this case, \eqref{eq:data_score} reduces to $({1}/{\rho^2})[\mathbb{E}_{\tilde{r}_{\phi^-}(z_0|x_0)}[\mathcal{D}_{\phi^-}(z_0)]-x_0] \approx ({1}/{\rho^2})[\mathcal{D}_{\phi^-}(\mathcal{E}_{\phi}(x_0))-x_0]$, whose vector Jacobian product with the decoder gradient resembles the data consistency term $\nabla_{z_t} ||\mathbb{E}[z_0|z_t] - \mathcal{E}(\mathcal{D}(\mathbb{E}[z_0|z_t]))||^2$ introduced in \cite{routSolvingLinearInverse2023,songSolvingInverseProblems2024}.
\paragraph{Adaptive Optimizer} To accelerate convergence on the challenging optimization landscapes common in imaging inverse problems, we employ a non-standard discretization for the flow in \eqref{eq:ode_flow_main}. Instead of a simple Euler step, we treat the drift terms as gradients and apply the Adam optimizer \citep{kingma2017adammethodstochasticoptimization}. We discuss this choice in \Cref{app:numerics}.

\section{Experiments}
\label{sec:experiments}
We evaluate our method \gls*{DWGF} on the FFHQ dataset \citep{karras2019stylebasedgeneratorarchitecturegenerative} downsampled to a resolution of $512\times 512$. We compare with \gls*{PSLD} \citep{routSolvingLinearInverse2023}, one of the established baselines for solving inverse problems with a latent diffusion prior and \gls*{RLSD} \citep{zilbersteinRepulsiveLatentScore2024}, which is a recent method based on similar ideas from gradient flows. We adopt same the experimental settings as in \cite{zilbersteinRepulsiveLatentScore2024} and report the best results therein. In this work, we only evaluate our method on box inpainting and super-resolution. We report standard metrics: \gls*{PSNR}, \gls*{LPIPS} \citep{zhang2018unreasonableeffectivenessdeepfeatures}, and \gls*{FID} \citep{heusel2017gan}.
\begin{table*}[h]
\centering \setlength{\tabcolsep}{4pt}
\resizebox{0.8\textwidth}{!}{
\begin{tabular}{lcccccc}
 & \multicolumn{3}{c}{\textbf{Inpainting (Box)}} & \multicolumn{3}{c}{\textbf{SR ($\times$8)}} \\ 
 \cmidrule(lr){2-4} \cmidrule(lr){5-7}
\textbf{Method} & \textbf{FID $\downarrow$} & \textbf{PSNR $\uparrow$} & \textbf{LPIPS $\downarrow$} & \textbf{FID $\downarrow$} & \textbf{PSNR $\uparrow$} & \textbf{LPIPS $\downarrow$}\\ \hline
PSLD \citep{routSolvingLinearInverse2023}     &   57.70  & 22.72   & 0.082   & 81.31   & 24.82   &   0.314  \\
RLSD \citep{zilbersteinRepulsiveLatentScore2024} &   29.18  & 24.98   & 0.079   &  65.42    &  28.39    &  0.286   \\ \hline
\textbf{Ours} &  118.05   &  27.56  &  0.184   &  101.94   &  32.71  &  0.193   \\
\bottomrule
\end{tabular}
}
\caption{Results for large box inpainting and $8\times$ super-resolution, all with noise $\sigma_y = 0.001$ on the FFHQ-512 validation dataset. 
}
\label{tab:512_comparison_FFHQ}
\end{table*}

\begin{figure}[h]
    \centering
    
    \begin{subfigure}{0.6\linewidth}
        \includegraphics[width=\linewidth]{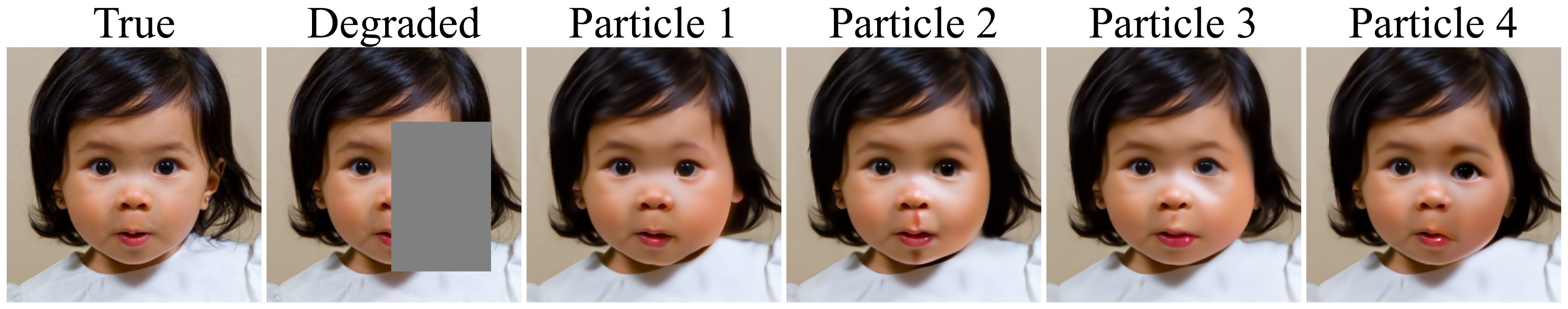}
        \caption{Diversity of the particles $x_0$ produced by DWGF on large box inpainting.}
        \label{fig:diverse}
    \end{subfigure}

    \begin{subfigure}{0.6\linewidth}
        \includegraphics[width=\linewidth]{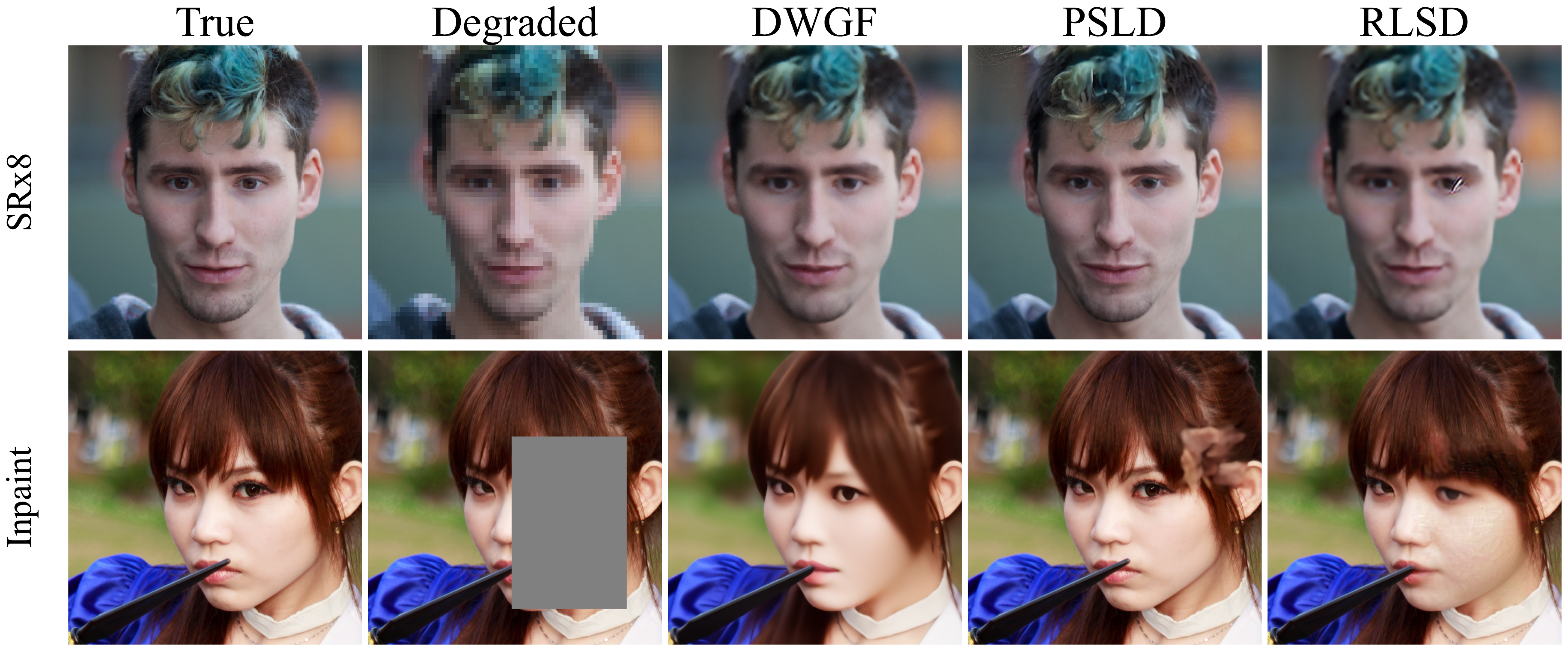}
        \caption{Qualitative comparison of the DWGF, PSLD, and RLSD on inpainting and super-resolution (8$\times$) tasks.}
        \label{fig:qualitative}
    \end{subfigure}
    
    \caption{Qualitative results on the FFHQ dataset downsampled to $512\times 512$ resolution.}
    \label{fig:combined} 
\end{figure}

As shown in the quantitative results (\Cref{tab:512_comparison_FFHQ}), our algorithm achieves comparable performance in terms of \gls*{PSNR} and \gls*{LPIPS} but suffers from poor \gls*{FID}. Examining the qualitative results (\Cref{fig:diverse,fig:qualitative}), we conjecture that this is caused by the absence of an entropic term in~\eqref{eq:min_functional}: the data-fidelity functional $\mathcal{F}$ places no reward on the spread of $q_\mu$, leading to blurry reconstructions. We point out that additional regularizations such as entropy \citep{wangTamingModeCollapse2024} or repulsive potentials \citep{corsoParticleGuidanceNoniid2024,zilbersteinRepulsiveLatentScore2024,huKernelDensitySteering2025} can be incorporated into our functional in \eqref{eq:min_functional} to tackle mode collapse.

\section{Conclusion}
We proposed \gls*{DWGF}, a novel approach of solving inverse problems using latent diffusion models as priors. We note that our approach can be extended to the conditional setting where we jointly optimize the prompt embedding as in \cite{spagnolettiLATINOPROLAtentConsisTency2025} using Euclidean-Wasserstein gradient flows \citep{kuntz23a}. Additionally, it would be interesting to explore control variates \citep{wangsteindreamervariancereduction2024} for variance reduction. 

\section*{Acknowledgements}
We thank the anonymous reviewers for their constructive comments. TW is supported by the Roth
Scholarship from the Department of Mathematics, Imperial College London.

\bibliographystyle{apalike}  
\bibliography{bibliography} 

\newpage
\appendix
\section{Derivations}
\label{app:derivation}
\textbf{Notations.} We denote the space of probability measures on $\mathbb{R}^d$ with finite $q$-th moments as $\mathcal{P}_q(\mathbb{R}^d)$. We equip the space $\mathcal{P}_2(\R^d)$ with the usual Wasserstein-2 scalar product \citep{figalli2021invitation} and denote its tangent space as $T\mathcal{P}_2(\R^d)$. For simplicity, we assume all distributions considered in this work admit  differentiable densities with respect to the Lebesgue measure.

\subsection{Derivation of the first variation of \texorpdfstring{\(\mathcal{F}\)}{F}}\label{app:grad_flow_derivation_F}
Lemma~\ref{lemma:first_var_kl} and Proposition~\ref{prop:kl_latent} below concern the KL divergence and are used in~\Cref{app:grad_flow_derivation_R}; Lemma~\ref{lemma:semi-implicit} gives the first variation of the linear functional $\mathcal{F}$ directly. We first state a standard result on the first variation of the KL divergence (cf. \citet{villani2003, santambrogio2016euclideanmetricwasserstein, figalli2021invitation}):
\begin{lemma}\label{lemma:first_var_kl}
    For $p, q \in \mathcal{P}_2(\mathbb{R}^d)$, the first variation of the KL divergence is given by:
    \begin{equation}
        \frac{\delta \KL(q\|p)}{\delta q} (x) = \log q(x) - \log p(x) +1, \quad \forall x\in \R^d
    \end{equation}
\end{lemma}
\begin{proof}
    Using the equivalent definition of first variation, we can write for $m \in T\mathcal{P}_2(\R^d)$ and $t>0$:
    \begin{equation}
        \KL(q+tm\|p) = \KL(q\|p) + t\left\langle m, \frac{\delta \KL(q\|p)}{\delta q} \right\rangle  + o(t),
    \end{equation}
    where the inner product is defined $\langle m,f\rangle:=\int_{\R^d} f(z)m(z) \mathrm{d} z$ for all $f,m\in T\mathcal{P}_2(\R^d)$.
    Using Taylor expansion of $(z+t)\log(z+t)=z\log z+t(\log z+1)+o(t)$, we can write $\KL(q+tm\|p)$ as:
    \begin{align}
        \KL(q+tm\|p) &= \int (q(x)+tm(x)) [\log (q(x)+tm(x))-\log p(x)] \dd x \\
        &=\int q(x)\log q(x) \dd x - \int q(x) \log p(x) \dd x \\
        &+ t\int (\log q(x) - \log p(x) +1) m(x)\dd x   + o(t),
    \end{align}
    which shows the desired result by matching the terms.
\end{proof}
\begin{lemma}\label{lemma:semi-implicit}
    For $\mathcal{F}: \mathcal{P}_2(\sZ)\to \R_+$ with $\mathcal{F}[\mu] :=\int w(z)\mu(z) \dd z$ (cf. \Cref{sec:method}) for any fixed $w(z):\R^d \to \R$, we have:
    \begin{equation}
        \frac{\delta\mathcal{F}[\mu]}{\delta \mu}(z) = w(z)
    \end{equation}
\end{lemma}
\begin{proof}
    Similar to the proof above, for any $m \in T\mathcal{P}_2(\sZ)$ and $t\in \R$, we have:
    \begin{align}
        \mathcal{F}[\mu+tm] &= \int w(z) (\mu(z) + tm(z)) \dd z \\
        &= \int w(z) \mu(z) \dd z + t \int w(z) m(z) \dd z.
    \end{align}
    From the definition of the functional derivative, this implies the desired result.
\end{proof}
\begin{proposition}\label{prop:kl_latent}
    For $q_\mu(x_0|y) = \int p(x_0|z_0) \mu(z_0|y) \dd z_0$ and $p(x_0|y)$ both in $\mathcal{P}_2(\mathbb{R}^d)$ we have:
    \begin{equation}
        \frac{\delta \KL(q_\mu(x_0|y)\|p(x_0|y))}{\delta \mu} = \mathbb{E}_{p(x_0|z_0)}[\log q_\mu(x_0|y) - \log p(x_0|y) + 1]
    \end{equation}
\end{proposition}
\begin{proof}
    Using the chain rule of functional derivatives \citep[Appendix A.3]{EngelEberhard2011DFTA} and the two lemmas above, we obtain:
    \begin{align}
        \frac{\delta \KL(q_\mu(x_0|y)\|p(x_0|y))}{\delta \mu} &= \int \frac{\delta \KL(q_\mu(x_0|y)\|p(x_0|y))}{\delta q}(x) \cdot \frac{\delta q_\mu(x_0|y)}{\delta \mu}(z_0) \dd x \\
        &= \int [\log q_\mu(x_0|y)-\log p(x_0|y)+1] p(x_0|z_0) \dd x,
    \end{align}
    where we have set $w(z_0)=p(x_0|z_0)$ in \Cref{lemma:semi-implicit} as the Gaussian decoder distribution.
\end{proof}

We now derive the gradient of the first variation.
\begin{proof}[Proof of\eqref{eq:F_gradient_of_first_variation}]
Substituting~\eqref{eq:semi_implicit_posterior} into the definition of $\mathcal{F}$ and exchanging the order of integration, we obtain
\begin{equation}
    \mathcal{F}[\mu] = -\int_\sX \log p(x_0|y)\int_\sZ p(x_0|z_0)\mu(z_0|y)\dd z_0\dd x_0
    = \int_\sZ V(z_0)\mu(z_0|y)\dd z_0,
\end{equation}%
where $V(z_0):=-\mathbb{E}_{p(x_0|z_0)}[\log p(x_0|y)]$, so the first variation follows from Lemma~\ref{lemma:semi-implicit} with $w=V$: \((\delta \mathcal{F}[\mu]/\delta \mu )(z_0) = V(z_0) = -\mathbb{E}_{p(x_0|z_0)}[\log p(x_0|y)]\). Assuming standard regularity conditions that permit exchanging expectation and gradient, the reparameterization $g_{\phi^-}(\epsilon,z_0):=\mathcal{D}_{\phi^-}(z_0)+\rho\epsilon$ and the chain rule yield
\begin{align}
    \nabla_{z_0}\frac{\delta \mathcal{F}[\mu]}{\delta \mu}
    &= -\nabla_{z_0}\mathbb{E}_{\epsilon\sim\mathcal{N}(0,I)}\left[\log p(g_{\phi^-}(\epsilon,z_0)|y)\right] \\
    &= -\mathbb{E}_{\epsilon\sim\mathcal{N}(0,I)}\left[\nabla_{x_0}\log p(g_{\phi^-}(\epsilon,z_0)|y)\frac{\partial\mathcal{D}_{\phi^-}(z_0)}{\partial z_0}\right],
\end{align}%
which is~\eqref{eq:F_gradient_of_first_variation}.
\end{proof}

\subsection{Derivation of the first variation of \texorpdfstring{\(\mathcal{R}\)}{R}}\label{app:grad_flow_derivation_R}
Similar to \Cref{lemma:first_var_kl} above, we can derive the first variation of the weighted KL divergence \eqref{eq:weighted_KL}. Recall that the weighted KL divergence for $p, q \in \mathcal{P}_2(\mathbb{R}^d)$ is given by:
    \begin{equation}
        \KL^{w, [0,T]}(q\|p) = \int_{[0,T]} w(s) \KL(q_s\| p_{s}) \dd s,
    \end{equation}
where both $q_s$ and $p_s$ are defined as pushforwards through a Markov kernel: $q_s(z_s):= \int p(z_s|z_0)q(z_0) \dd z_0$, where we define $p(z_s|z_0):=\mathcal{N}(z_s;\alpha_s z_0, \sigma_s^2I)$.

\begin{proposition}\label{prop:first_var_weighted_kl}
     For $p, q \in \mathcal{P}_2(\mathbb{R}^d)$, the first variation of the weighted KL divergence is given by:
    \begin{equation}
        \frac{\delta \KL^{w, [0,T]}(q\|p)}{\delta q} (x_0) = \int_{[0,T]}w(s)\mathbb{E}_{p(z_s|z_0)}[\log q_s (z_s) - \log p_s(z_s)+1] \dd s, \quad \forall z_0\in \R^d
    \end{equation}    
\end{proposition}
\begin{proof}
    Applying the chain rule for functional derivatives again, we obtain:
    \begin{align}
        \frac{\delta \KL^{w, [0,T]}(q\|p)}{\delta q} (z_0) &= \int \frac{\delta \KL^{w, [0,T]}(q\|p)}{\delta \KL(q_s\| p_{s})} (s) \frac{\delta \KL(q_s\| p_{s})}{\delta q}(z_0) \dd s \\
        &= \int w(s) \mathbb{E}_{p(z_s|z_0)}[\log q_s(z_s)  - \log p_s(z_s)+1] \dd s
    \end{align}
    where we have used \Cref{lemma:semi-implicit} to obtain the first variation of the first term. To obtain the second term ${\delta \KL(q_s\| p_{s})}/{\delta q}$, we use the proof of \Cref{prop:kl_latent} by replacing $p(x_0|z_0)$ with $p(z_s|z_0)$, $q_\mu(x_0|y)$ with $q_s(z_s)$, and $p(x_0|y)$ with $p(z_s)$.
\end{proof}

We can now derive the gradient of the first variation in \eqref{eq:R_gradient_of_first_variation}.
\begin{proof}[Proof of \eqref{eq:R_gradient_of_first_variation}]
    Substituting the definition for $p=p_{\theta^-}(z_0),q=\mu(z_0|y)$ with $p(z_s|z_0)=\mathcal{N}(z_s;\alpha_s z_0, \sigma_s^2I)$ into \Cref{prop:first_var_weighted_kl} and taking the gradient yields:
    \begin{align}
        \nabla_{z_0} \frac{\delta \mathcal{R}[\mu]}{\delta \mu} &= \nabla_{z_0}\int_{[0,T]}  w(s) \mathbb{E}_{p(z_s|z_0)}\left[
        \left(\log \mu(z_s|y) - \log p_{\theta^-}(z_{s})\right) 
    \right] \dd s  \\
    &=\int_{[0,T]} w(s) \nabla_{z_0} \mathbb{E}_{\epsilon\sim \mathcal{N}(0,I)}\left[
        \left(\log \mu(g_s(\epsilon, z_0)|y) - \log p_{\theta^-}(g_s(\epsilon, z_0))\right) 
    \right] \dd s  \\
    &=\int_{[0,T]} w(s) \mathbb{E}_{\epsilon\sim \mathcal{N}(0,I)}\left[
        \left(\nabla_{z_0} \log \mu(g_s(\epsilon, z_0)|y) - \nabla_{z_0} \log p_{\theta^-}(g_s(\epsilon, z_0))\right) 
    \right] \dd s  \\
        &=\int_{[0,T]} w(s)  \mathbb{E}_{\epsilon\sim \mathcal{N}(0,I)}\left[
        \left(\nabla_{z_s} \log \mu(g_s(\epsilon, z_0)|y) - \nabla_{z_s} \log p_{\theta^-}(g_s(\epsilon, z_0))\right) \frac{\partial z_{s}}{\partial z_{0}}
    \right],
    \end{align}
where the second line follows from the reparameterization trick $g_s(\epsilon, z_0) = \alpha_s z_0 + \sigma_s \epsilon$, the third line follows from an exchange of the gradient and the integral, and we used the chain rule in the final line. 

The marginal density $\mu(g_s(\epsilon,z_0)|y)=\mu(z_s|y)$ can be approximated using an integral $\int_\sZ p(z_{s}|z_{0}) \mu(z_{0,t}|y) \dd z_{0}$, whence \eqref{eq:R_gradient_of_first_variation} follows.
\end{proof}


\subsection{Proof to \texorpdfstring{\Cref{thm:weighted_kl_properties}}{Theorem \ref*{thm:weighted_kl_properties}}}
\label{app:proof_thm_kl}
We restate the theorem here for convenience:
\begin{theorem}
    The weighted KL divergence $\KL^{w, [0,T]}(\mu(z_0|y)\| p_{\theta^-}(z_0))$ \eqref{eq:weighted_KL} is i) nonnegative, ii) convex in the first component $\mu(z_0|y)$, and iii)  is minimized if and only if the standard KL divergence $\KL(\mu(z_0|y)\|p_{\theta^-}(z_0))$ is minimized.
\end{theorem}
\begin{proof}
    We note that the first two properties follows from those of the KL divergence. For the third property, we note that the weighted KL is zero if and only if $\KL(\mu(z_t|y)\|p_{\theta^-}(z_t))=0$ for every $t$, since the weights $w(t)$ are nonnegative. Now using the argument in \citet{wangProlificDreamerHighFidelityDiverse2023}, we see that $\mu(z_t|y)=p_{\theta^-}(z_t)$ if and only if their characteristic functions are equal $\varphi_{\mu(z_t|y)}(s)=\varphi_{p_{\theta^-}(z_t)}(s)$. But we have:
    \begin{align}
        \varphi_{\mu(z_t|y)}(s) &= \varphi_{\mu(z_0|y)}(\alpha_t s) \cdot \varphi_{\mathcal{N}(0,I)}(\sigma_ts) \\
        \varphi_{p_{\theta^-}(z_t)}(s) &= \varphi_{p_{\theta^-}(z_0)}(\alpha_t s) \cdot \varphi_{\mathcal{N}(0,I)}(\sigma_ts),
    \end{align}
    since we have from the forward process $z_t = \alpha_t z_0+\sigma_t \epsilon$ for $\epsilon \sim \mathcal{N}(0,I)$. The third property thus follows from the positivity of the KL divergence.
\end{proof}

\section{Algorithm}
We now give the pseudocode for simulating the gradient flow in \Cref{algo:diff_kl} below; we use the notation $z_{s,k}$ to denote the particles diffused to the time $s\in [0,T]$ and at step $k \in \mathbb{N}$ of the Euler discretization of the ODE \eqref{eq:ode_flow_main}. Note that instead of approximating the drift corresponding to the diffusion regularization in \eqref{eq:R_gradient_of_first_variation} by sampling time $s$ uniformly from $[0,T]$, we follow \citet{mardaniVariationalPerspectiveSolving2023} and adopt a deterministic schedule for sampling $s$ to be $s\in \{T, T-1, \cdots, 0\}$. In our experiments, we take $w(s)= c\sigma_s^2/\alpha_s$ with a constant $c\in(0,1)$ similarly to \citet{mardaniVariationalPerspectiveSolving2023} .

\begin{algorithm}[h]
\caption{Diffusion-regularized Wasserstein Gradient Flow (DWGF)}\label{algo:diff_kl}
\begin{algorithmic}[1]
\State \textbf{Inputs:} Observation $y \in \sY$, differentiable forward operator $\mathcal{A}:\sX\to \sY$, initial particles $\{z^{(i)}_{0,0}\}_{i=1}^N$, diffusion weights $w(t)$, regularization strength $\gamma$, data consistency weight $\lambda$, pretrained latent diffusion model with score $s_{\theta^-}(\cdot, \cdot):\sZ \times \R_+ \to \sZ$ and its \vae~with encoder-decoder pair $(\mathcal{E}_{\phi^-},\mathcal{D}_{\phi^-})$
\State Set counter $k\gets 0$
\For{$s\in \{T, T-1, \cdots, 0\}$}
    \State Sample $\epsilon^{(i)} \sim \mathcal{N}(0,I)$ and decode $x^{(i)}_0\gets g_{\phi^-}(\epsilon, z_{0,k}^{(i)})$ for $i\in [N]$
    \State Data likelihood $\nabla_{x_0^{(i)}}\ell(x_{0}^{(i)})\gets \nabla_{x_0^{(i)}} (-\frac{1}{2\sigma_y^2}\|y-\mathcal{A}(x_{0}^{(i)})\|^2_2)$
    \State \begin{align}    
    u(z_{0,k}^{(i)})&\gets \left( - \frac{\lambda}{\rho^2}[\mathcal{D}_{\phi^-}(\mathcal{E}_{\phi^-}(x_0^{(i)}))]-x_0^{(i)}]-\nabla_{x_0^{(i)}}\ell(x_{0}^{(i)})\right)\frac{\partial \mathcal{D}_{\phi^-}(z_{0,k}^{(i)})}{\partial z_{0,k}^{(i)}}
    \end{align}
    \State Sample $\nu^{(i)}\sim \mathcal{N}(0,I)$ for all $i\in [N]$
    \State Compute $z_{s,k}^{(i)}\gets \alpha_s z_{0,k}^{(i)}+\sigma_s\nu^{(i)}$
    \State \begin{align}
        v(z_{0,k}^{(i)})\gets w(s) \left(\nabla_{z_{s,k}^{(i)}} \log \frac{1}{N}\sum_{j=1}^N \mathcal{N}(z_{s,k}^{(i)}; \alpha_s z_{0,k}^{(j)}, \sigma_s^2I) -  s_{\theta^-}(z_{s,k}^{(i)}, s)\right) \frac{\partial z_{s,k}^{(i)}}{\partial z_{0,k}^{(i)}}
    \end{align}
    \State Simulate $z_{0,k+1}^{(i)}\gets \mathrm{OptimizerStep}(z_{0,k}^{(i)}, u(z_{0,k}^{(i)})+ \gamma v(z_{0,k}^{(i)}))$ for all $i\in [N]$
    \State $k\gets k+1$
\EndFor
\State\Return Particles decoded $\{\mathcal{D}_{\phi^-} (z^{(i)}_{0,k})\}_{i=1}^N$
\end{algorithmic}
\end{algorithm}

\section{Experimental Details}
\label{app:numerics}
\paragraph{Implementation} In practice, we use the Adam optimizer \citep{kingma2017adammethodstochasticoptimization} with a learning rate of $lr=1.0$ and default hyperparameters $(\beta_1,\beta_2)=(0.9,0.999)$ to solve the ODE as in \Cref{algo:diff_kl}; this is done by viewing $u(z_{0,k}^{(i)})+ \gamma v(z_{0,k}^{(i)})$ as the gradient of a loss function to be optimized. This approach effectively introduces momentum and an adaptive diagonal preconditioning at each step, which helps mitigate ill-conditioning and flat minima. While unconventional for numerical solutions to ODEs, this heuristic is supported by recent work that formally incorporates momentum \citep{limmomentumparticlemaximum2024} and preconditioning \citep{limParticleSemiimplicitVariational2024} into the simulation of gradient flows.

We set the balancing coefficient $\gamma=0.15$ and the data consistency weight to $\lambda=0.1 \rho^2$ (cf. \Cref{sec:algo_consi}), where we choose the decoder standard deviation to be $\rho=10^{-2}$. We choose Stable Diffusion v2.1 \citep{rombachhighresolutionimagesynthesis2022} as the base model and set the number of sampling steps $T=999$, which are the same as \citet{zilbersteinRepulsiveLatentScore2024}.

\paragraph{Evaluation} We follow the experimental setups in \citet{zilbersteinRepulsiveLatentScore2024} and evaluate our methods on the first $100$ images from the validation set of FFHQ \citep{karras2019stylebasedgeneratorarchitecturegenerative} using $4$ particles. Due to limited computational budget, we do not reproduce the experiments in \citet{zilbersteinRepulsiveLatentScore2024} but choose to report their best results for each task therein, namely the non-repulsive version of \gls*{RLSD}.

\section{Further Discussion and Related Works}
\label{app:further_discussion}
\paragraph{Further Discussions} Despite potentials of \gls*{DWGF} in inverse problems, it has some significant limitations. Firstly, our method does not yet match the performance of state-of-the-art in terms of FID (\Cref{sec:experiments}), which we attribute to the lack of proper regularization and insufficient hyperparameter tuning. Secondly, analogous to other variational approaches \citep{mardaniVariationalPerspectiveSolving2023,zilbersteinRepulsiveLatentScore2024}, \gls*{DWGF} requires a large number of sampling steps. A promising research direction is thus to integrate our framework with recent advances in few-step and consistency models~\citep{songconsistencymodels2023,luoDiffinstructUniversalApproach2024,yinImprovedDistributionMatching2024,geng2025meanflowsonestepgenerative}, which are emerging contenders to traditional iterative models. Finally, our method uses a particle cloud to approximate the intractable integral in solving the Wasserstein gradient flow, which may lead to considerable memory requirements.

\paragraph{Related Works} Gradient flow approaches have been adopted in the context of general inverse problems in recent approaches \citep{akyildiz2025efficientpriorcalibrationindirect,vadeboncoeur2025efficientdeconvolutionpopulationalinverse}. In imaging, \citet{hagemann2024posteriorsamplingbasedgradient} uses a gradient flow of the \mmd~\citep{jmlr:v13:gretton12a}, which is parametrized as a sequence of pushforward maps. \cite{zilbersteinRepulsiveLatentScore2024} uses an interacting particle system approach with repulsive potential, but they compute the marginal score of the latent distribution $\nabla_{z_t} \log q(z_t^{(i)}|y)$ as $- ({z}_t^{(i)} - \alpha_t {z}_0^{(i)}) / \sigma_t^2 = -{\epsilon^{(i)}}/{\sigma_t}$, which only holds when $q(z_0|y)$ is Gaussian, hence their approach reduces (approximately) to a Bures-Wasserstein gradient flow, similar to RED-Diff \citep{mardaniVariationalPerspectiveSolving2023}. Our work is most related to \citet{wangProlificDreamerHighFidelityDiverse2023}, which treat the parameters of an 3D representation MLP as particles to be optimized via Wasserstein gradient flow. However, their method operates in a conditional setting and involves training another network to approximate the score. Also related to our approach are works on score-distillation, which involve the training of a generator $p_\theta$ to approximate the output of a diffusion model \citep{yinOnestepDiffusionDistribution2024,yin2024improveddistributionmatchingdistillation,luoDiffinstructUniversalApproach2024,xieEMDistillationOnestep2024}. 
\end{document}